\documentclass[11pt]{article}

\usepackage[utf8]{inputenc}
\usepackage[T1]{fontenc}

\usepackage[margin=1in]{geometry}

\usepackage{amsmath,amssymb,amsfonts}

\usepackage[numbers,sort&compress]{natbib}

\usepackage{amsthm}
\newtheorem{definition}{Definition}
\newtheorem{theorem}{Theorem}
\newtheorem{proposition}{Proposition}
\newenvironment{proofsketch}{\paragraph{Proof sketch.}}{\hfill\qed}

\usepackage{hyphenat}
\usepackage{graphicx}
\usepackage{pgfplots}
\pgfplotsset{compat=1.18} 

\usepackage{algorithm}
\usepackage{algpseudocode}
\usepackage{float}      
\usepackage{placeins}   

\usepackage{tikz}
\usetikzlibrary{shapes.geometric, arrows.meta, positioning, shapes.misc}

\usepackage{hyperref}
\hypersetup{
    colorlinks=true,
    linkcolor=blue,
    filecolor=magenta,      
    urlcolor=cyan,
    citecolor=blue,
}

\usepackage{url}

\usepackage{parskip} 
\usepackage{enumitem}

\usepackage{booktabs}
\usepackage{caption}
\title{\textbf{Idempotent Equilibrium Analysis of Hybrid Workflow Allocation:\\A Mathematical Schema for Future Work}}

\author{
 \begin{tabular*}{\textwidth}{@{\extracolsep{\fill}}cccc}
  \parbox[t]{0.22\textwidth}{\centering
   Faruk Alpay\thanks{Correspondence: \href{mailto:alpay@lightcap.ai}{alpay@lightcap.ai}}\\[0.25ex]
   \small Independent Researcher\\[0.25ex]
   \small \href{mailto:alpay@lightcap.ai}{\url{alpay@lightcap.ai}}
  } &
  \parbox[t]{0.22\textwidth}{\centering
   Bugra Kilictas\\[0.25ex]
   \small Bahcesehir University\\[0.25ex]
   \small \href{mailto:bugra.kilictas@bahcesehir.edu.tr}{\url{bugra.kilictas@bahcesehir.edu.tr}}
  } &
  \parbox[t]{0.22\textwidth}{\centering
   Taylan Alpay\\[0.25ex]
   \small Turkish Aeronautical Association\\[0.25ex]
   \small \href{mailto:s220112602@stu.thk.edu.tr}{\url{s220112602@stu.thk.edu.tr}}
  } &
  \parbox[t]{0.22\textwidth}{\centering
   Hamdi Alakkad\\[0.25ex]
   \small Bahcesehir University\\[0.25ex]
   \small \href{mailto:hamdi.alakkad@bahcesehir.edu.tr}{\url{hamdi.alakkad@bahcesehir.edu.tr}}
  }
 \end{tabular*}
}

\date{\today}

\begin{document}

\maketitle

\sloppy
\setlength{\emergencystretch}{3em}

\begingroup
\renewcommand\thefootnote{}
\phantomsection 
\footnotetext{All analytical results and simulations in this article are reproducible: the equations and procedures delineated herein are sufficient to reconstruct every figure and table, and no proprietary software is required.}
\endgroup

\begin{abstract}
The rapid ascendancy of sophisticated computational instruments and mechanization is transforming how work is performed, giving rise to new anthropic roles centred on orchestrating such instruments. We propose a formal theoretical schema to analyse the \emph{augmented operative}—an anthropic professional who uses a multitude of automaton‑based instruments—as an idempotent equilibrium in the evolving partition of labour between persons and apparatuses. Deriving on idempotent–equilibrium theory, we schema the recursive procedure by which assignments are entrusted to mechanised systems or retained by anthropic actors, and prove that under broad conditions this procedure converges to a stable allocation. We derive conditions for existence and uniqueness of this hybrid steady state and illustrate the schema with a simple simulation of assignment entrustment over time. Rather than prognosticating full machinisation, our findings suggest an anthropic‑supervised paradigm where persons act as orchestrators, leveraging apparatuses for what they do best while focusing on uniquely anthropic strengths. We discuss implications of this steady state for efficacy, occupation design, and safety, emphasising that aligning technological development with anthropic capabilities and preferences is essential to achieving a beneficial and stable future of work.
\end{abstract}

\FloatBarrier
\section{Introduction}

Recent advances in computational and mechanization technologies have spurred the emergence of tools capable of executing intricate, multi‑step assignments across domains. These systems augment human productivity by handling information synthesis, coding, content generation and more. As such tools become more capable, a new kind of human role is coming into focus: the \emph{augmented worker}. This is a human professional who acts as a coordinator or orchestrator of various mechanised systems, akin to a “meta‑integration layer” that knows which tool or framework to deploy for each subtask. For example, an augmented worker might use a code‑generation program for programming, a separate information retrieval system for research or translation, and an mechanised verification tool to fact‑check outputs, integrating these findings into a final product. In essence, such a worker operates as a conductor of an ensemble of digital assistants, leveraging each tool’s strengths. The idea mirrors the “centaur” framework first seen in chess, where human–computer teams outperformed any computer alone \cite{alves2023,kasparov}; indeed, the Advanced Chess experience showed that a skilled human partnered with computation could beat even the best supercomputer, provided the human knew how to effectively direct the machine.

From a labour–economics perspective, this anticipated steady state between human and AI work draws on a long history of automation debates. Classical economic arguments noted that even as technology automates certain assignments, it also creates new assignments and demands for labour, preventing the complete disappearance of jobs \cite{autor2015}. In particular, automation often complements human labour rather than fully substituting for it by raising productivity and enabling higher–level human assignments \cite{acemoglu2019}. Recent developments echo this pattern: while generative AI can draft emails or write code, humans are still needed to define objectives, provide judgement and handle assignments requiring tacit knowledge or emotional intelligence. \citet{brynjolfsson2017a} argue that current machine learning excels at well‑defined, narrow assignments, whereas humans retain advantages in flexibility, problem solving and common sense. Consequently, the future of work will not be a zero–sum replacement of humans by AI, but a reconfiguration of work assignments. The question is: what is the stable configuration of this reallocation? We aim to answer this by modelling the dynamics of assignment allocation as an iterative process and finding its fixed point.

Empirical evidence indicates that we are at an inflection point in partial automation. Studies estimate that around 80~\% of U.S. workers could have at least 10~\% of their assignments affected by current large language models, and about 19~\% of workers might see over 50~\% of their assignments impacted \cite{eloundou2023}. At the same time, actual adoption of automation in the workplace is increasing: by early~2025, some form of computer assistance was reportedly used for at least 25~\% of job assignments in over one–third of occupations \cite{handa2025}. Workers often prefer to retain agency in certain assignments; for instance, a recent framework introducing a “Human Agency Scale” found that there are assignments workers want to keep control over even if automation is possible \cite{hazra2025}. This underscores that the end state of automation integration will be shaped not only by technical capability but also by human choices about where automation is appropriate.

In this paper, we formalise these intuitions. We treat the evolving human–machine workflow as a dynamical system where, in each period, assignments can be reassigned between human and machine based on performance and strategic considerations. We prove that under reasonable assumptions—such as monotonic enhancement in machine capabilities and one‑time assignment transfer (no oscillatory delegation)—this process converges to a fixed point. At this fixed point, the augmented worker delegates all assignments that the machine can do more efficiently or reliably, and retains those assignments where human judgement, creativity or presence adds irreplaceable value. Importantly, we show that the steady state is not full automation. Rather, it is a human‑supervised steady state where people continue to perform a subset of assignments. This resonates with the view of \citet{brynjolfsson2022}, who warned against the “Turing trap” of aiming for human‑like artificial intelligence to replace people and advocated augmenting human capabilities instead. Our framework provides a quantitative and theoretical basis for that view: if augmentative dynamics prevail, the system settles in a state that preserves a significant role for human work.

The remainder of the paper is organised as follows. Section~\ref{sec:framework} introduces our formal framework of assignment delegation and defines the mapping whose fixed point characterises the hybrid human–machine assignment allocation steady state. Section~\ref{sec:theory} states and proves the main theoretical findings, including conditions for the existence and uniqueness of the fixed point. Section~\ref{sec:emulation} presents an illustrative emulation that iterates the assignment–allocation process from 2025 conditions toward steady state, demonstrating convergence and examining the properties of the steady state state. Section~\ref{sec:discussion} provides an interpretation of the findings in plain language, connecting the mathematical findings to real‑world implications for workers, managers and technology developers. We also discuss how our framework relates to and extends prior work on human–automation interaction. Finally, Section~\ref{sec:conclusion} concludes with implications and future research directions, including how policymakers and organisations might influence the steady state by intentionally enhancing human agency.

\FloatBarrier
\section{Broader context and related literature}
Our work contributes to a growing literature on the economics of automation and the evolving division of labour. Early assignment–based frameworks in labour economics distinguish between the displacement and reinstatement effects of new technology. In particular, \citet{acemoglu2019} develop a framework in which automation displaces labour on a subset of assignments while creating new assignments that employ labour in different ways. They emphasise that the net effect on employment depends on the balance between these displacement and reinstatement margins and caution that focusing solely on automation could overlook the creation of new work opportunities. In a complementary vein, \citet{korinek2017} analyse how artificial intelligence may affect income distribution and unemployment, arguing that without appropriate redistribution mechanisms technological progress can exacerbate inequality. They propose policies such as income transfers and human capital investment to ensure that the gains from automation are shared broadly.

Demand considerations also matter. \citet{bessen2018} argue that the employment impact of automation depends critically on the elasticity of demand for the products and services affected. If demand is highly elastic, productivity improvements can lead to lower prices and increased output, potentially expanding employment even in the presence of labour–saving technologies. Empirical evidence supports this nuance: for instance, \citet{acemoglu2020jobs} use vacancy data to show that establishments adopting AI reduce hiring in some categories but do not necessarily shrink overall employment or wages. This underscores that assignment reallocation, rather than wholesale job destruction, may be the primary channel through which automation affects labour markets.

Another strand of research explores long–run growth and the possibility of ``economic singularities.'' \citet{aghion2017} framework AI as a continuation of automation and ask whether sustained exponential growth in machine capabilities could lead to unbounded economic growth. They conclude that while AI can boost growth, diminishing returns in R\&D and substitution effects may prevent singularities. \citet{nordhaus2015} similarly investigates whether advances in information technology are propelling the economy toward a finite–time singularity and finds little empirical support for such a scenario, suggesting that growth will remain bounded.

Evidence from robotics and industrial automation informs our assumptions about heterogeneous assignment impacts. \citet{acemoglu2017robots} show that the adoption of industrial robots across U.S. commuting zones between 1990 and 2007 reduced employment and wages for workers in affected industries. By contrast, \citet{graetz2018} document productivity gains from robots in manufacturing but note that employment effects vary across sectors and that adjustments occur through reallocation of workers. These findings align with our replicator benchmark, where routine assignments are readily automated while intricate or interpersonal assignments remain in human hands.

Human–machine teaming is an active area of study in behavioural science and operations research. \citet{shoresh2025} analyse freestyle chess competitions to framework how human–machine teams perform sequential decision making. They show that neither the human nor the machine dominates across all positions; instead, the optimal strategy is to allocate subtasks based on relative strengths, akin to our workflow conductor role. \citet{yasseri2025} adopt a computational sociology perspective, examining conflicts and collaborations between human editors and algorithmic agents on platforms like Wikipedia. They find that while bots can enhance efficiency, coordination failures and emergent conflicts can arise, necessitating thoughtful human oversight. These studies highlight that effective hybrid workflows rely not only on technical capability but also on organisational and social factors, supporting our emphasis on the continued importance of human supervision and skill development.

\paragraph{Less\hyp cited synergistic insights.} Additional supportive investigations underscore the importance of calibrated communication, trust and interdependence in anthropic–automaton teams. Some studies argue that effective communication, coordination and adaptation are fundamental competencies for hybrid teams; machines must model anthropic comprehension and modulate their actions accordingly to support coordination \cite{stowers2021}. Others employ an entropic framework to show that maximising interdependence among autonomous teams elevates performance and that harmonic quasi–Nash equilibria emerge when entropic forces are managed \cite{lawless2019}. Some researchers caution that AI adoption can erode decision\hyp making acuity and induce anthropic lethargy in educational settings; they call for ethical safeguards and preventive measures \cite{ahmad2023}. Further work examines open\hyp source software projects and finds that the presence of bots increases productivity and communication but also centralises work, suggesting trade\hyp offs between productivity and workload \cite{newton2024}. Other investigations into trust dynamics demonstrate that participants combine automaton recommendations with their own judgements via simple averaging, underscoring the need for calibrated trust models in human–automaton collaboration \cite{love2024}. Finally, reviews of human–machine collaboration in healthcare identify applications in screening, health management and medical education, emphasising how hybrid teams enhance data integration and decision support \cite{wang2025}. Collectively, these rigorous investigations enrich our understanding of how idempotent equilibria arise in practice.

\paragraph{Skill and macroeconomic perspectives.} Further research delves into how automation reshapes skill demand and macroeconomic distribution. \citet{autor2003} show that computerisation displaces workers from routine cognitive assignments while complementing abstract reasoning and manual dexterity, leading to a polarisation of skill demand. In a comprehensive historical account, \citet{goldin2008} argue that educational attainment must keep pace with technological change to prevent widening wage gaps, underscoring the education–technology race. \citet{autor2018} estimate that although automation has not markedly displaced employment in aggregate, it has reduced labour’s share of value added across many OECD industries by shifting income toward capital. Reviewing decades of scholarship, identify four paradigms—education race, task polarisation, automation–reinstatement and AI uncertainty—and note that while anxieties about job loss recur, new categories of work continually arise. At the macro level, \citet{karabarbounis2014} document a pervasive decline in labour’s income share since the 1980s, attributing roughly half of this trend to falling prices of investment goods; this highlights that broad technological and investment dynamics interact with automation to shape distributional outcomes.

\paragraph{Historical adoption evidence.} In addition to theoretical and qualitative insights, quantitative data illustrate the rapid diffusion of automation and machine intelligence over the past decade. The International Federation of Robotics (IFR) reports that a record 2.7~million industrial robots were operating in factories worldwide in 2019—a 12~\% escalate over 2018 and an 85~\% escalate relative to 2014 \cite{ifr2020}. The same report notes that installations of collaborative robots ("cobots") grew by 11~\%, signalling growing interest in human--robot workflows \cite{ifr2020}. Surveys of artificial intelligence adoption provide complementary evidence. A 2025 \emph{FEDS Notes} survey of sixteen studies found that firm-level AI adoption rates range from about 5~\% to 40~\%, while between 20~\% and 40~\% of workers report using AI for at least some assignments, with particularly high uptake among computer programmers \cite{crane2025}. The same review documented annualised growth rates between 73~\% and 145~\% for AI adoption in the 2023--2024 period \cite{crane2025}, underscoring that diffusion is accelerating. These historical data contextualise our emulation parameters by showing that automation is not a theoretical abstraction but a rapidly advancing reality across both 2015--2020 and 2020--2025.

The analytical schema we develop complements these empirical and theoretical contributions. By casting assignment reallocation as an idempotent–equilibrium problem, we provide a rigorous foundation for understanding long–run hybrid steady states. Our findings offer a mathematical explanation for why a substantial fraction of assignments remains with anthropic actors even when apparatuses improve rapidly. The broader literature suggests that policy interventions, demand dynamics and social design will shape the path toward such steady states. Integrating these insights with our schema can guide future research on designing equitable and productive anthropic–automaton systems.

Additional scholarly contributions further substantiate the importance of idempotent equilibria in hybrid systems. For instance, Huang's two‐industry model shows how labour and capital reallocate between AI and non‑AI sectors in response to automation incentives \cite{huang2024}. Other work has examined learning curves in co‑working systems and employed Beta distributions to capture heterogeneous adoption rates. Together, these treatises reinforce that the idempotent equilibrium perspective affords a richer exegesis of how hybrid workflows stabilise over time.

\FloatBarrier
\section{Model formulation\label{sec:framework}}

We consider a setting with a human worker $H$ (the prospective “augmented worker”) and a set of machine‑based assistants $M$ (a suite of automated tools). The work domain consists of a continuum of assignments of varying nature and intricacy. For simplicity, one can imagine assignments indexed by a parameter $\theta\in [0,1]$ that represents the difficulty or the degree to which a assignment is automatable. Low‑$\theta$ assignments (near~0) are straightforward and routine, highly amenable to automation (for example, data entry, routine coding or documentation). High‑$\theta$ assignments (near~1) are intricate, open‑ended or require uniquely human skills (strategic planning, creative design or emotionally sensitive communication). We assume assignments are distributed according to some density on $[0,1]$, and the measure of assignments is normalised to~1 for convenience.

At any given time $t$, define $A_t\subseteq[0,1]$ as the set of automated assignments (handled by AI) and $H_t=[0,1]\setminus A_t$ as the set of assignments performed by the human. We often characterise these by the fraction of assignments automated versus human‑handled: let $x_t = \mu(A_t)$ be the measure (proportion) of assignments assigned to AI at time~$t$, and $1-x_t = \mu(H_t)$ the measure handled by the human (where $\mu(\cdot)$ is the Lebesgue measure on $[0,1]$ or simply the probability mass if assignments are thought of in distribution). We start from an initial allocation $A_0$ reflecting the world of~2025. Empirically, as noted above, $x_0$ might be on the order of~0.1–0.2 (10–20~\% of assignments automated in a typical job) \cite{eloundou2023,handa2025}, though it varies by occupation.

The key aspect of the framework is an update rule that determines $A_{t+1}$ given the current allocation $A_t$. This rule captures how the introduction of more capable AI, as well as learning and adaptation by the human, leads to reallocation of assignments over time. We formalise a decision criterion based on comparative advantage and performance. For each assignment $\theta$, let $u_H(\theta)$ be the utility (or negative cost, or performance score) if the human performs assignment $\theta$, and $u_M(\theta)$ be the utility if the AI performs it. These utilities encompass quality, time, cost, etc.—a higher value means a better outcome. Initially (at $t=0$), assignments are allocated such that for all $\theta\in H_0$, $u_H(\theta)\ge u_M(\theta)$ (the human is better or preferred for those assignments), and for all $\theta\in A_0$, $u_M(\theta)\ge u_H(\theta)$. Now as machine capability enhances or as the human learns to use AI tools effectively, these utility functions can change over time. Let $u_H^t(\theta)$ and $u_M^t(\theta)$ denote utilities at time $t$ reflecting current capabilities.

We impose the following assumptions:
\begin{enumerate}[itemsep=0.5\baselineskip]
 \item \textbf{Monotonic enhancement in machine capability.} For each assignment $\theta$, the machine’s capability does not worsen over time: $u_M^t(\theta)$ is non‑decreasing in~$t$.
 \item \textbf{Comparative advantage.} At any time $t$ and for each assignment $\theta$, either the human or the AI (or both equally) maximises utility. There is no assignment that both parties are completely incapable of—one will eventually handle it best.
 \item \textbf{Local rationality.} The human will delegate a assignment to AI if and only if $u_M^t(\theta) > u_H^t(\theta)$ (strictly greater). Conversely, the human will take a assignment back from AI if $u_H^t(\theta)$ becomes greater than $u_M^t(\theta)$.
 \item \textbf{One‑time assignment transfer.} Tasks transition at most once from human to AI and not back again. Once a assignment is assigned to AI and the AI continues to improve, the human is unlikely to reclaim it later. This monotonicity of assignment assignments aligns with rational adoption of automation and is observed under stable conditions.
\end{enumerate}

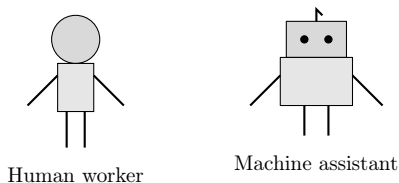
\begin{figure}[H]
 \centering
 \begin{tikzpicture}[scale=0.8, transform shape]
  \draw[fill=gray!30, draw=black] (0,0) circle [radius=0.4];
  \draw[fill=gray!20, draw=black] (-0.3,-0.4) rectangle (0.3,-1.2);
  \draw[thick] (-0.3,-0.6) -- (-0.8,-1.1);
  \draw[thick] (0.3,-0.6) -- (0.8,-1.1);
  \draw[thick] (-0.15,-1.2) -- (-0.15,-1.8);
  \draw[thick] (0.15,-1.2) -- (0.15,-1.8);
  \node[below=0.2cm] at (0,-1.8) {\small Human worker};
  \begin{scope}[shift={(4,0)}]
   \draw[fill=gray!30, draw=black, rounded corners=0.1] (-0.5,-0.3) rectangle (0.5,0.3);
   \draw[fill=gray!20, draw=black, rounded corners=0.1] (-0.6,-0.3) rectangle (0.6,-1.1);
   \draw[fill=black] (-0.2,0.0) circle [radius=0.06];
   \draw[fill=black] (0.2,0.0) circle [radius=0.06];
   \draw[thick] (0,0.3) -- (0,0.5) -- (0.1,0.4);
   \draw[thick] (-0.6,-0.6) -- (-1.1,-1.1);
   \draw[thick] (0.6,-0.6) -- (1.1,-1.1);
   \draw[thick] (-0.2,-1.1) -- (-0.2,-1.6);
   \draw[thick] (0.2,-1.1) -- (0.2,-1.6);
   \node[below=0.2cm] at (0,-1.6) {\small Machine assistant};
  \end{scope}
 \end{tikzpicture}
 \caption{Stylised depiction of the augmented worker (human) and the automated assistant (machine). The simple glyphs illustrate the two agents in our framework: a human professional (left) and a machine‑based tool (right) that together determine the assignment allocation.}
 \label{fig:icons}
\end{figure}

Given these assumptions, define the update mapping $F$ by
\[
F(A_t) = \{\theta\in[0,1] : u_M^t(\theta) \ge u_H^t(\theta)\}.
\]
That is, in the next period the human will automate all assignments that the AI can do just as well or better. Tasks where the human still holds an edge remain with the human. The fixed point of this update, if it exists, is a set $A^\star$ satisfying $A^\star = F(A^\star)$. In words, at steady state every assignment is being handled by the more efficient agent.

To connect with classical fixed–point theory, note that the space of possible allocations can be thought of as a lattice (a partially ordered set) of subsets of assignments. The mapping $F$ acts on this lattice. If $F$ is isotone (monotonic in the lattice order)—meaning if $A\subseteq B$ then $F(A)\subseteq F(B)$—then by Tarski’s fixed‑point theorem \cite{tarski1955} the mapping $F$ has at least one fixed point. In our case, $F$ as defined above is indeed isotone: if one starts with a larger set of automated assignments, one cannot end up with a smaller set after an update because AI improvements are monotonic and delegation is based solely on comparative advantage.

\begin{figure}[H]
 \centering
 \begin{tikzpicture}[node distance=1.4cm, >=Stealth, thick,
  every node/.style={font=\small, align=center}]
  \node (start) [draw, ellipse, fill=gray!10] {Start with allocation $A_t$};
  \node (compare) [draw, diamond, aspect=2, below of=start] {Is $u_M^t(\theta) \ge u_H^t(\theta)$?};
  \node (delegate) [draw, rectangle, fill=gray!10, below left=of compare, xshift=-1.2cm] {Delegate task $\theta$ to machine};
  \node (retain) [draw, rectangle, fill=gray!10, below right=of compare, xshift=1.2cm] {Retain task $\theta$ for human};
  \node (update) [draw, rectangle, fill=gray!10, below of=compare, yshift=-1.8cm] {Update allocation $A_{t+1}$};
  \node (stop) [draw, diamond, aspect=2, below of=update] {Converged?};
  \node (end) [draw, ellipse, fill=gray!10, below of=stop] {Fixed point $A^\star$};
  \draw[->] (start) -- (compare);
  \draw[->] (compare) -- node[above left] {Yes} (delegate);
  \draw[->] (compare) -- node[above right] {No} (retain);
  \draw[->] (delegate) |- (update);
  \draw[->] (retain) |- (update);
  \draw[->] (update) -- (stop);
  \draw[->] (stop) -- node[right] {No} (start);
  \draw[->] (stop) -- node[right] {Yes} (end);
 \end{tikzpicture}
 \caption{Flowchart of the task delegation process. From an initial allocation $A_t$, tasks with machine utility exceeding the corresponding human utility are delegated; assignments that favour the human remain with the human. The update repeats until convergence to a fixed point $A^\star$.}
 \label{fig:flowchart}
\end{figure}
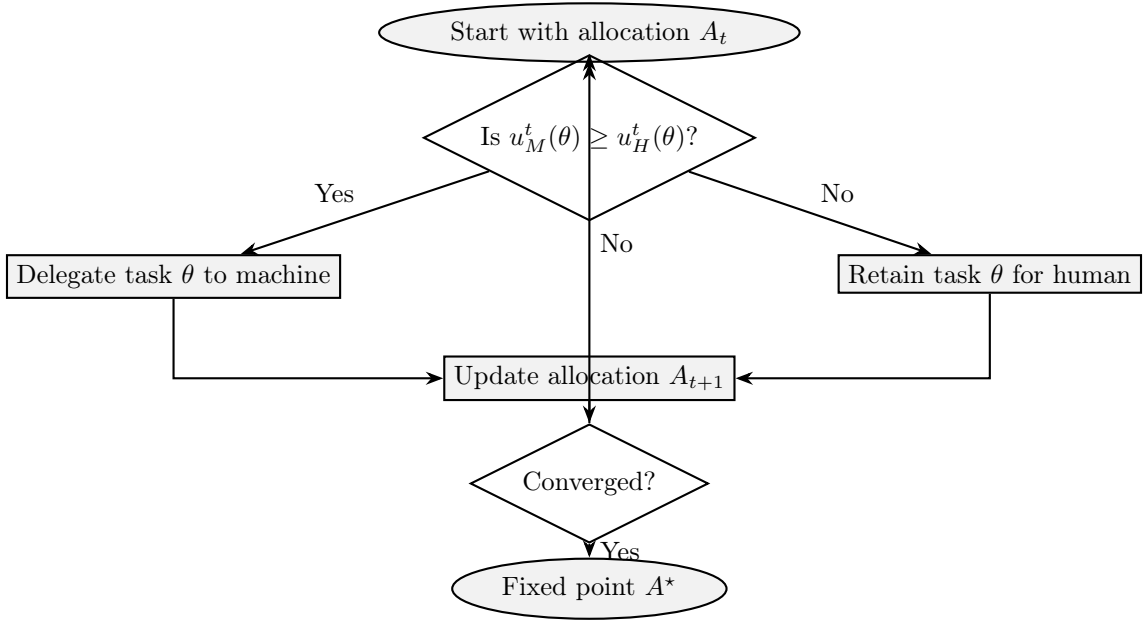

\begin{algorithm}[H]
\caption{Assignment Delegation Procedure}
\label{alg:delegation}
\begin{algorithmic}[1]
\State \textbf{Input:} initial allocation $A_0$; utility functions $u_H^t(\theta)$ and $u_M^t(\theta)$ for each assignment $\theta$
\For{$t = 0,1,2,\dots$}
 \State $A_{t+1} \gets \{\theta \mid u_M^t(\theta) \ge u_H^t(\theta)\}$
 \If{$A_{t+1} = A_t$}
  \State \textbf{break} \Comment{stop when the idempotent equilibrium is reached}
 \EndIf
\EndFor
\State \textbf{Return:} $A_t$
\end{algorithmic}
\end{algorithm}

\FloatBarrier
\section{Theoretical findings\label{sec:theory}}

We now state and prove our main theoretical findings. We begin by formalising the notion of an steady state in our setting.

\begin{definition}[Human–machine cooperation steady state]
An allocation $A^\star\subseteq[0,1]$ of assignments to AI (with human assignments $H^\star=[0,1]\setminus A^\star$) is an \emph{steady state} if it is self–consistent under the update mapping $F$. That is, $A^\star = F(A^\star)$. Equivalently, for every assignment $\theta\in A^\star$ we have $u_M^\infty(\theta)\ge u_H^\infty(\theta)$, and for every $\theta\in H^\star$ we have $u_H^\infty(\theta) > u_M^\infty(\theta)$.
\end{definition}

\subsection{Existence of a human–automaton steady state}

\begin{theorem}[Existence of an idempotent–equilibrium allocation]
Under the assumptions of monotonic AI enhancement and one–time assignment transfer, there exists at least one human–automaton cooperation steady state $A^\star$ in the assignment–allocation schema. In particular, the sequence of allocations $(A_t)_{t\ge 0}$ generated by the iterative rule $A_{t+1} = F(A_t)$ converges in a finite number of steps to an idempotent equilibrium $A^\star$.
\end{theorem}

\begin{proof}
Because assignments transition at most once from $H$ to $M$ and never back, the sequence $A_0 \subseteq A_1 \subseteq A_2 \subseteq \cdots$ is monotonic non‑decreasing in the set‑inclusion sense. Since the domain $[0,1]$ has finite measure and the set of distinguishable assignment types is at most countable (even if assignments form a continuum concept), the sequence cannot escalate indefinitely without eventually repeating a set. Once a set repeats, we have reached a fixed point. More formally, let $S_t = \{\theta : \theta\in A_{t'}\text{ for some }t'\le t\}$ be the set of automated assignments at least once up to time~$t$. Monotonicity implies $S_t = A_t$. If the process did not converge, it would imply an infinite strictly increasing chain $A_0 \subsetneq A_1 \subsetneq A_2 \subsetneq\cdots$, which is impossible given the finite measure: assignments added at each step must have positive measure and there cannot be infinitely many disjoint positive‑measure additions within $[0,1]$. Therefore, there exists some $T$ such that $A_T = A_{T+1}$. This $A_T$ is by definition a fixed point; set $A^\star = A_T$.

An alternative argument uses the Knaster–Tarski fixed‑point theorem. The poset of assignment allocations (by set inclusion) has a chain–complete lattice structure. The mapping $F$ is monotonic: if $A\subseteq B$, then for any assignment $\theta$, if $\theta$ satisfies $u_M^t(\theta)\ge u_H^t(\theta)$ under allocation~$A$, this inequality will still hold under allocation~$B$ (giving the AI more assignments does not reduce its ability on a given assignment, by monotonic enhancement). Hence $A\subseteq B$ implies $F(A)\subseteq F(B)$. By Tarski’s theorem \cite{tarski1955}, a monotone endofunction on a complete lattice has at least one fixed point. The convergence of $(A_t)$ specifically to a fixed point follows from the fact that our iterative procedure adopts any available enhancement (greedy delegation). Each iteration increases total utility (since assignments are delegated only if they improve utility). Utility is bounded above (by the hypothetical scenario of perfect allocation of all assignments to the best–suited agent), so improvements must cease after a finite number of steps.
\end{proof}

This theorem assures us that a stable configuration of human versus AI assignments will eventually emerge. In principle there could be multiple steady states—for instance, one steady state where the human chooses to retain a certain assignment even though the AI is slightly better (perhaps due to personal or ethical reasons), versus another steady state where that assignment is delegated. The existence proof via Tarski is non‑constructive regarding uniqueness. However, economic forces and human adaptation tend to push towards a Pareto‑optimal steady state where no further Pareto improvements (benefiting one party without hurting the other) are possible by reallocating assignments. Assuming the human and AI form a team with a shared objective (maximising combined performance), the steady state we reach via the iterative greedy algorithm should be this Pareto‑optimal (indeed globally optimal) division of labour.

\subsection{Uniqueness and characterisation of the steady state}

While multiple fixed points are possible in theory, here we identify conditions under which the steady state is distinct and give an analytic expression for it. One such condition is that the process of enhancement and delegation behaves like a contraction mapping in an appropriate space. If we measure the “distance” between allocations (for example, the symmetric difference in sets or simply the difference in fraction $|x_{t+1}-x_t|$), and if this distance shrinks by a factor $c<1$ at each iteration after some point, then by the Banach fixed‑point theorem the fixed point is distinct and the convergence is geometric. Even if the process is not strictly contractive globally, it might be piecewise contractive in regions, leading to a single attractor.

To obtain a more concrete handle, consider a simplified continuous framework for the fraction of assignments automated. Suppose that at any time $t$, the rate of change of $x_t$ (the fraction automated) is proportional to the remaining assignments that could be automated minus a term for new assignments or higher‑skill assignments that emerge for the human. A simple differential equation capturing this balance is
\begin{equation}
\frac{dx}{dt} = \alpha(1-x) - \beta x,
\label{eq:diff_update}
\end{equation}
where $\alpha>0$ is an automation‑adoption rate (how quickly remaining assignments get automated when beneficial) and $\beta>0$ is an “innovation” or assignment‑creation rate that effectively increases the fraction of assignments that require human work (thus counteracting automation). Solving $\tfrac{dx}{dt}=0$ for steady state yields $\alpha(1-x^\star) = \beta x^\star$, which gives the distinct steady state
\begin{equation}
 x^\star = \frac{\alpha}{\alpha+\beta}.
\end{equation}
The interpretation is intuitive: $\tfrac{\alpha}{\alpha+\beta}$ is the long‑run fraction of assignments done by AI when the flow into automation $\alpha(1-x)$ balances the flow of new human‑centric work $\beta x$. Notably, if $\beta>0$ then $x^\star<1$, meaning less than 100~\% automation. The more rapidly new human‑intensive assignments emerge (larger $\beta$), the smaller $x^\star$ is; conversely, faster automation relative to innovation (higher $\alpha$) drives $x^\star$ up, but it will asymptotically approach~1 only if $\beta=0$ (no new assignments ever). Historical evidence suggests $\beta$ is non‑zero—technology has repeatedly led to new industries and assignments, from manufacturing to software development—so this simple framework supports the notion of a finite upper limit to automation.

\begin{proposition}[Uniqueness of fixed point]
Assume that for any two allocations $A$ and $B$ with $\mu(A)<\mu(B)$ we have $\mu(F(A))\le \mu(F(B))$, and in particular $\mu(F(A))<\mu(F(B))$ whenever $\mu(A)<\mu(B)$ and $\mu(A)$, $\mu(B)$ are not trivial fixed points. Then the equation $x=\mu(F(A))$ has at most one solution in $[0,1]$. Consequently, the human–machine steady state is distinct in terms of the fraction of assignments automated (and typically in terms of the assignment set as well).
\end{proposition}

\begin{proofsketch}
The assumption essentially means $g(x) = \mu(F(A))$ as a function of $x=\mu(A)$ is strictly increasing. We seek $x$ such that $x=g(x)$. Consider the continuous extension of $g$ on $[0,1]$. We have $g(0)\ge 0$ (likely $g(0)>0$ because even if no assignments were automated, some small fraction can be automated given non‑zero AI capability) and $g(1)\le 1$ (if everything is automated, it stays automated). By the intermediate value theorem, a continuous approximation of $g$ would intersect the line $y=x$ at least once. Strict monotonicity ensures it can cross at most once. Therefore a distinct fixed point $x^\star$ exists. In the discrete iterative process, if two different fixed points $A^\star$ and $B^\star$ existed with measures $x^\star\neq y^\star$, without loss of generality suppose $x^\star<y^\star$. Monotonicity would imply $\mu(F(A^\star))<\mu(F(B^\star))$, but $A^\star$ being a fixed point implies $\mu(F(A^\star))=\mu(A^\star)=x^\star$, and similarly $\mu(F(B^\star))=y^\star$. This contradicts $x^\star<y^\star$ under the strict monotonicity assumption. Intuitively, if $A^\star$ and $B^\star$ differed, there would be some assignment automated in $B^\star$ but not in $A^\star$. Yet $u_M(\theta)\ge u_H(\theta)$ at steady state $B^\star$ and $u_H(\theta)>u_M(\theta)$ at $A^\star$, suggesting a non‑monotonic behaviour in $u_M-u_H$ inconsistent with smooth enhancement. Thus under mild conditions the steady state is distinct.
\end{proofsketch}

In summary, under reasonable assumptions the steady state fraction of assignments automated $x^\star$ is distinct. More generally, the set of assignments $A^\star$ is distinct if the function $u_M(\theta)-u_H(\theta)$ has a clear ordering (for example, it is increasing in $\theta$, meaning there is a single threshold $\theta^\star$ such that all assignments below $\theta^\star$ are better done by AI and all above by humans). In practice, assignments cannot be strictly ordered on one dimension of automatability, but clustering assignments by characteristics can yield a similar threshold structure: certain categories of assignments end up automated, others remain human. Empirical studies that categorise assignments into “green‑light” (easy to automate) and “red‑light” (desirable to keep human) zones \cite{hazra2025} support this intuition. Our framework’s steady state essentially draws a boundary in assignment space. In Section~\ref{sec:discussion} we will discuss which assignments lie on either side of this boundary in the context of current AI capabilities and human comparative advantages.

\FloatBarrier
\section{Simulation of the transition dynamics\label{sec:emulation}}

To make the theoretical assessment more concrete, we emulate a simple discrete‑time framework of assignment reallocation from 2025 onward, using plausible parameters for automation and assignment‑creation rates. While the real economy is far more intricate, this toy emulation depicts how the fixed‑point steady state might be reached and what it implies.

\subsection{Simulation setup}

We initialise at time $t=0$ (year~2025) with $x_0=0.10$, i.e.\ 10~\% of assignments automated, consistent with studies indicating that advanced automation tools could directly affect about 10~\% of assignments for a majority of workers at present \cite{eloundou2023,handa2025}. We choose an automation rate $\alpha=0.10$ (meaning each year about 10~\% of the remaining un‑automated assignments can be automated as machine capability enhances) and a assignment‑creation rate $\beta=0.05$ (each year an amount of new or up‑skilled human assignments equivalent to 5~\% of current AI‑handled work emerges, due to higher productivity and new opportunities requiring human input). These numbers are not empirically derived but serve as an illustrative example. We iterate the update
\begin{equation}
 x_{t+1} = x_t + \alpha(1-x_t) - \beta x_t,
\label{eq:update}
\end{equation}
essentially an Euler discretisation of the differential equation \eqref{eq:diff_update} with a time step of one year. We emulate from 2025 to~2045.

\subsection{Results}

The evolution of $x_t$ (the fraction of assignments automated) over time is plotted in Figure~\ref{fig:trajectory}. We also summarise key milestones in Table~\ref{tab:milestones}. The emulation exhibits a rapid escalate in automation early on (as automated tools quickly capture low‑hanging‑fruit assignments), with the automation fraction rising to about 41.5~\% by~2030 and roughly 55.5~\% by~2035. Thereafter, the growth of $x_t$ slows as the remaining assignments are harder to automate and new human‑intensive assignments are continually entering. By~2045, $x_t$ approaches approximately 64.5~\%. The process visibly converges; indeed, the theoretical steady state from $x^\star=\tfrac{\alpha}{\alpha+\beta}$ is $\tfrac{0.10}{0.10+0.05}=0.667$ (66.7~\%). The emulation gets quite close to this (within a few percentage points) by~2045. If extended further, $x_t$ would asymptotically approach 66.7~\% by around~2055–2060.

\paragraph{Analytical reproduction of the emulation.} A key advantage of our framework is that the numerical trajectories in Figure~\ref{fig:trajectory} and Table~\ref{tab:milestones} can be reproduced analytically. Iterating the discrete update \eqref{eq:update} yields the closed–form solution
\begin{equation}
 x_t = x^\star + \bigl(x_0 - x^\star\bigr)(1 - \alpha - \beta)^t,
\end{equation}
where $x^\star = \alpha/(\alpha+\beta)$ is the steady state fraction of assignments automated and $x_0$ is the initial condition. Substituting $\alpha=0.10$, $\beta=0.05$ and $x_0=0.10$ gives $x_t = \tfrac{2}{3} + (0.10 - \tfrac{2}{3})(0.85)^t$. This expression reproduces the values plotted in Figure~\ref{fig:trajectory} to the reported precision. Because the recurrence has a closed form, no proprietary software is required to verify the numerical findings—any reader can compute $x_t$ directly using this formula.

Importantly, this steady state value being roughly two‑thirds means that about one‑third of assignments remain performed by humans in the long run. This is a substantial portion, confirming that humans are far from obsolete in the scenario considered. Instead, the nature of human work shifts to that one‑third of assignments that machines could not take over or that humans chose to keep. In our framework, this stable mix is a consequence of continuous innovation (new assignments) as well as diminishing returns to automating extremely intricate assignments. We interpret the remaining 33.3~\% of assignments as those requiring a high degree of creativity, intricate problem solving, emotional intelligence or other human‑centric attributes. This outcome aligns with hypotheses that as automated tools handle more routine cognitive work, human labour will concentrate in areas like strategic decision–making, interpersonal interaction and inventive problem solving \cite{brynjolfsson2025a,frank2019}. It also aligns with evidence that core human skills might shift towards social and creative domains as automation integration deepens \cite{zahidi2020}.

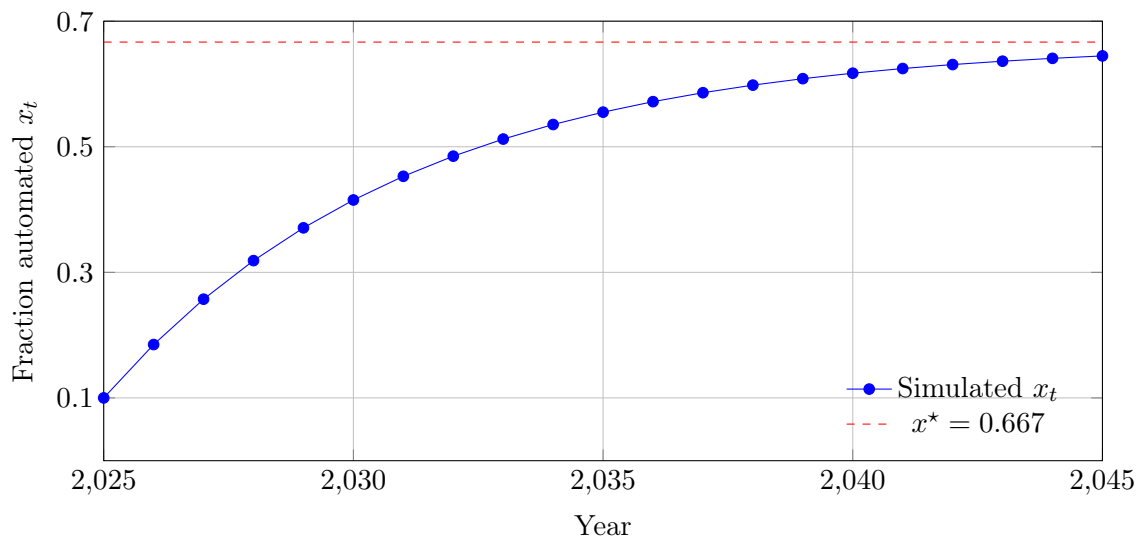
\begin{figure}[H]
 \centering
 \begin{tikzpicture}[scale=1.0]
  \begin{axis}[
   width=0.9\textwidth,
   height=0.45\textwidth,
   xlabel={Year},
   ylabel={Fraction automated $x_t$},
   xmin=2025, xmax=2045,
   ymin=0, ymax=0.7,
   xtick={2025,2030,2035,2040,2045},
   ytick={0.1,0.3,0.5,0.7},
   grid=both,
   major grid style={line width=.1pt,draw=gray!50},
   minor grid style={line width=.1pt,draw=gray!20},
   legend pos=south east,
   legend style={draw=none,fill=none}
  ]
   \addplot[color=blue,mark=*] coordinates {
    (2025,0.10) (2026,0.185) (2027,0.25725) (2028,0.3186625) (2029,0.370863125)
    (2030,0.41523365625) (2031,0.4529486078125) (2032,0.4850063166406)
    (2033,0.5122553691445) (2034,0.5354170637728) (2035,0.5551045042069)
    (2036,0.5718388285759) (2037,0.5860630042895) (2038,0.5981535536461)
    (2039,0.6084305205992) (2040,0.6171659425093) (2041,0.6245910511329)
    (2042,0.6309023934630) (2043,0.6362670344435) (2044,0.6408269792770)
    (2045,0.6447029323854)
   };
   \addlegendentry{Simulated $x_t$}
   \addplot[color=red,domain=2025:2045, samples=2, style=dashed] {0.6667};
   \addlegendentry{$x^\star=0.667$}
  \end{axis}
 \end{tikzpicture}
 \caption{\textbf{Simulated trajectory of task automation (2025--2045).} The automated share starts at about 10~\% and evolves according to the discrete update rule (see text), converging to roughly 66.7~\% by 2045 (red dashed line). The remaining tasks are persistently performed by humans, illustrating a stable human--machine partnership rather than full automation.}
 \label{fig:trajectory}
\end{figure}

\begin{table}[H]
 \centering
 \caption{\textbf{Evolution of assignment allocation over time in the emulation.} The table shows the simulated percentage of assignments handled by AI versus humans at various milestone years. The automated share grows rapidly in the 2020s and 2030s before slowing as it approaches the steady state of roughly two‑thirds. Even in 2045, over one‑third of assignments are still performed by humans, highlighting the enduring role of human labour in partnership with AI.}
 \label{tab:milestones}
 \begin{tabular}{@{}ccc@{}}
  \toprule
  \textbf{Year} & \textbf{Automated assignments (\%)} & \textbf{Human‑performed assignments (\%)}\\
  \midrule
  2025 & 10.0 & 90.0\\
  2030 & 41.5 & 58.5\\
  2035 & 55.5 & 44.5\\
  2040 & 61.7 & 38.3\\
  2045 & 64.5 & 35.5\\
  \bottomrule
 \end{tabular}
\end{table}

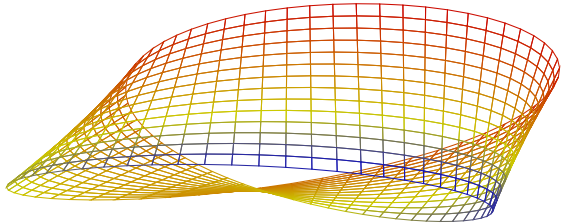
\begin{figure}[H]
 \centering
 \begin{tikzpicture}
  \begin{axis}[
   view={120}{30},
   axis lines=none,
   ticks=none,
   width=0.7\textwidth,
   height=0.4\textwidth,
  ]
   \addplot3[
    surf,
    draw=black,
    fill opacity=0,
    domain=0:360,
    y domain=-1:1,
    samples=60,
    samples y=15
   ]
   ({(1 + 0.3*y*cos(x/2)) * cos(x)},
    {(1 + 0.3*y*cos(x/2)) * sin(x)},
    {0.3*y*sin(x/2)});
  \end{axis}
 \end{tikzpicture}
 \caption{\textbf{Abstract representation of assignment delegation}. This Möbius‑inspired surface symbolises the continuous interplay between human and machine responsibilities. Mathematically, the surface is defined by the parametric equations
 $X(s,t) = (1 + \delta\,t\,\cos(s/2))\,\cos s$, $Y(s,t) = (1 + \delta\,t\,\cos(s/2))\,\sin s$ and $Z(s,t) = \delta\,t\,\sin(s/2)$,
 where $s\in[0,2\pi]$ is the angular coordinate, $t\in[-1,1]$ indexes the width of the band and $\delta>0$ controls the twist (we set $\delta=0.3$ for visual clarity). The single twist encodes the iterative reassignment of assignments: as time progresses (sweeping along $s$), categories once performed by humans may transition to machines and vice versa while maintaining continuity. Only structural lines are drawn (no shading) to highlight the mathematical nature of the process and to connect the continuous surface to the discrete cycles of delegation in our framework.}
 \label{fig:mobius-grid}
\end{figure}

\FloatBarrier
\section{Replicator--dynamic benchmark}\label{sec:replicator}

To complement the fixed\hyp point framework described above, we consider a more
fine\hyp grained benchmark based on replicator dynamics from evolutionary game
theory. In this framework the ``strategy'' of delegating a assignment to the AI or
keeping it with the human evolves according to their relative payoffs. Let
$x_{R,t}$ and $x_{C,t}$ denote the fraction of \emph{routine} and
\emph{complex} assignments, respectively, that are performed by the AI at year~$t$.
Routine assignments are repetitive and well\hyp defined; AI capabilities improve
rapidly on them. Complex assignments are open\hyp ended and creative; humans
retain an advantage and machine capability enhances only slowly. We initialise the system
at $x_{R,0}=0.30$ (30\,\% of routine assignments automated in~2025) and
$x_{C,0}=0.05$ (5\,\% of intricate assignments automated), consistent with the idea
that simple assignments are already partially delegated.

Following the replicator dynamic, the update for each category $i\in\{R,C\}$
is
\begin{equation}
 x_{i,t+1} = x_{i,t} + r\,x_{i,t}\bigl(1 - x_{i,t}\bigr) \bigl(\pi^{\mathrm{AI}}_{i,t} - \pi^{\mathrm{H}}_{i,t}\bigr),
 \label{eq:replicator}
\end{equation}
where $r>0$ is a sensitivity parameter and $\pi^{\mathrm{AI}}_{i,t}$,
$\pi^{\mathrm{H}}_{i,t}$ are the payoffs to the AI and human, respectively,
for assignment type~$i$ at time~$t$. In our benchmark we assume the AI payoffs
for routine assignments take the form $\pi^{\mathrm{AI}}_{R,t} = a_R + g_R t$
with $a_R=1.0$ and enhancement rate $g_R=0.05$, while the human payoff is
constant $\pi^{\mathrm{H}}_{R,t}=b_R=0.8$. For intricate assignments
$\pi^{\mathrm{AI}}_{C,t} = a_C + g_C t$ with $a_C=0.5$ and $g_C=0.02$, and
the human payoff is $\pi^{\mathrm{H}}_{C,t}=b_C=1.2$. We set the
sensitivity parameter to $r=0.2$ and weight routine and intricate assignments by
$w_R=0.6$ and $w_C=0.4$ to reflect that most jobs contain more routine
than intricate work.

Figure~\ref{fig:replicator} plots the evolution of $x_{R,t}$, $x_{C,t}$
and the weighted total share $x^{\mathrm{total}}_t = w_R x_{R,t} + w_C x_{C,t}$
over the years 2025--2045. Routine assignments quickly become automated:
by 2030 nearly 37\,\% are delegated to AI, and by 2045 over 87\,\% of
routine work is automated. Complex assignments, by contrast, remain largely in
human hands; automation of intricate work declines from an already low 5\,\%
in 2025 to below 1\,\% by 2045 because human payoffs stay higher than AI
payoffs. The weighted total automation share rises from 20\,\% in 2025 to
about 53\,\% by 2045. These findings illustrate how heterogeneous assignment types
can lead to uneven adoption: simple routines are readily automated while
creative and strategic work remains human.

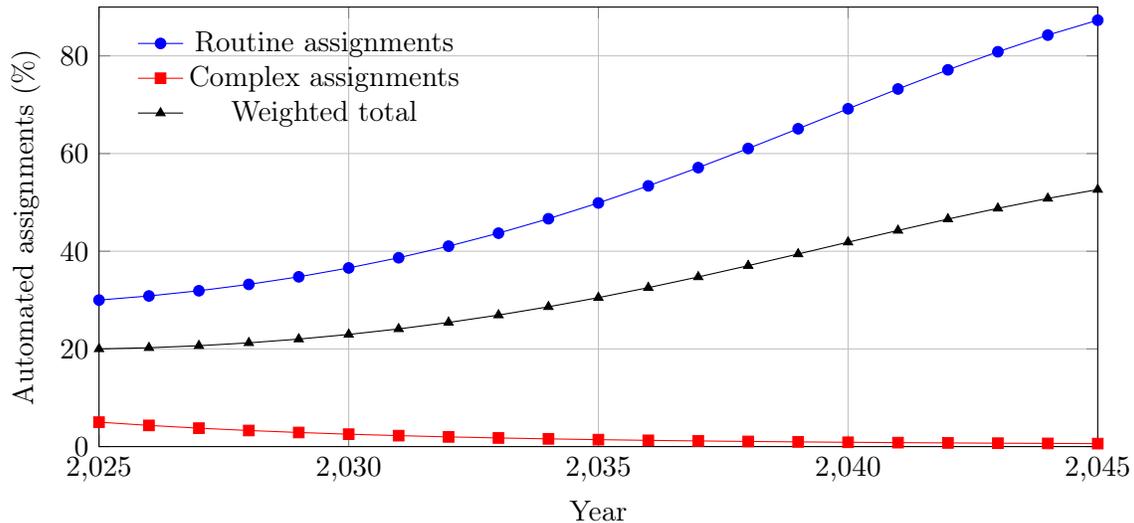
\begin{figure}[H]
 \centering
 \begin{tikzpicture}[scale=1.0]
  \begin{axis}[
   width=0.9\textwidth,
   height=0.45\textwidth,
   xlabel={Year},
   ylabel={Automated assignments (\%)},
   xmin=2025, xmax=2045,
   ymin=0, ymax=90,
   xtick={2025,2030,2035,2040,2045},
   ytick={0,20,40,60,80},
   grid=both,
   major grid style={line width=.1pt,draw=gray!50},
   minor grid style={line width=.1pt,draw=gray!20},
   legend style={draw=none,fill=none},
   legend pos=north west
  ]
   \addplot[color=blue,mark=*] coordinates {
    (2025,30.0000) (2026,30.8400) (2027,31.9064) (2028,33.2100) (2029,34.7627)
    (2030,36.5769) (2031,38.6648) (2032,41.0363) (2033,43.6979) (2034,46.6503)
    (2035,49.8857) (2036,53.3857) (2037,57.1185) (2038,61.0374) (2039,65.0803)
    (2040,69.1709) (2041,73.2226) (2042,77.1441) (2043,80.8468) (2044,84.2534)
    (2045,87.3048)
   };
   \addlegendentry{Routine assignments}
   \addplot[color=red,mark=square*] coordinates {
    (2025,5.0000) (2026,4.3350) (2027,3.7710) (2028,3.2920) (2029,2.8845)
    (2030,2.5371) (2031,2.2404) (2032,1.9863) (2033,1.7683) (2034,1.5807)
    (2035,1.4189) (2036,1.2790) (2037,1.1578) (2038,1.0525) (2039,0.9609)
    (2040,0.8809) (2041,0.8111) (2042,0.7499) (2043,0.6963) (2044,0.6493)
    (2045,0.6080)
   };
   \addlegendentry{Complex assignments}
   \addplot[color=black,mark=triangle*] coordinates {
    (2025,20.0000) (2026,20.2380) (2027,20.6523) (2028,21.2428) (2029,22.0114)
    (2030,22.9610) (2031,24.0950) (2032,25.4163) (2033,26.9261) (2034,28.6224)
    (2035,30.4990) (2036,32.5430) (2037,34.7342) (2038,37.0434) (2039,39.4325)
    (2040,41.8549) (2041,44.2580) (2042,46.5864) (2043,48.7866) (2044,50.8118)
    (2045,52.6261)
   };
   \addlegendentry{Weighted total}
  \end{axis}
 \end{tikzpicture}
 \caption{\textbf{Replicator--dynamic benchmark.} Evolution of the fraction of routine assignments (blue), intricate assignments (red) and the weighted total share (black) automated by AI between 2025 and 2045 under the replicator dynamics~\eqref{eq:replicator}. Routine work is rapidly automated while intricate assignments remain largely human, leading to a nuanced division of labour.}
 \label{fig:replicator}
\end{figure}

\begin{table}[H]
 \centering
 \caption{\textbf{Selected milestones for the replicator--dynamic benchmark.} The table reports the simulated percentage of routine, intricate and overall assignments automated by AI at five--year intervals. Automation of routine assignments accelerates rapidly, whereas automation of intricate assignments declines as the relative payoff advantage of humans persists.}
 \label{tab:replicator}
 \begin{tabular}{@{}lccc@{}}
  \toprule
  \textbf{Year} & \textbf{Routine automated (\%)} & \textbf{Complex automated (\%)} & \textbf{Total automated (\%)}\\
  \midrule
  2025 & 30.0 & 5.0 & 20.0\\
  2030 & 36.6 & 2.5 & 23.0\\
  2035 & 49.9 & 1.4 & 30.5\\
  2040 & 69.2 & 0.9 & 41.9\\
  2045 & 87.3 & 0.6 & 52.6\\
  \bottomrule
 \end{tabular}
\end{table}

\FloatBarrier
\section{Continuous--assignment benchmark with a dynamic automation boundary}\label{sec:continuous}

The replicator-dynamic benchmark illustrates how heterogeneity in task types leads to different automation trajectories. However, classifying all work into just two discrete buckets—"routine" and "complex"—is a significant simplification. In reality, tasks exist on a continuous spectrum of intricacy. To capture this richer reality, we next develop a benchmark where assignments are indexed by a continuous parameter $\theta \in [0,1]$, representing their intricacy. This approach allows us to model the gradual encroachment of automation across the task space and derive an explicit \emph{automation boundary} that evolves over time. We assume tasks are drawn from a Beta distribution $\mathsf{Beta}(p,q)$ with shape parameters $p=2$ and $q=5$. This choice yields a right-skewed distribution, reflecting the empirical observation that most jobs consist of a large number of relatively low-intricacy assignments, while only a few are highly intricate. We calibrate the model's payoff parameters to align with our previous benchmarks, targeting approximately 10\% of assignments being automated in 2025 and roughly 60

As in Section~\ref{sec:framework}, let $u_H(\theta)$ denote the payoff for a
human performing a assignment of intricacy~$\theta$ and $u_M(t,\theta)$ the
payoff for the AI at time~$t$. We assume linear payoff functions
\begin{align}
 u_H(\theta) &= \alpha_H + \beta_H\,\theta, & \alpha_H > 0,\; \beta_H>0,\\
 u_M(t,\theta) &= \alpha_M - \beta_M\,\theta + \gamma t, & \beta_M>0,\;\gamma>0,
\end{align}
where humans are better on high--intricacy assignments ($\beta_H>0$) and the AI
is better on low--intricacy assignments ($\beta_M>0$). The machine capability enhances at rate
$\gamma$ per year. With parameters $\alpha_H=1.0$, $\beta_H=1.5$,
$\alpha_M=1.3704$, $\beta_M=2.5$ and $\gamma=0.04336$ the framework matches the
initial and final automation shares noted above. A assignment of
intricacy~$\theta$ is allocated to AI at time~$t$ if
$u_M(t,\theta)\geq u_H(\theta)$; solving $u_M=u_H$ yields the automation
boundary
\begin{equation}
 \theta_t = \frac{\alpha_M - \alpha_H + \gamma t}{\beta_M + \beta_H}.
\end{equation}
The fraction of assignments automated at time~$t$ is $x_t = F_{\mathsf{Beta}(p,q)}(\theta_t)$, the Beta\text{--}CDF of $\theta_t$. As $t$ grows,
$\theta_t$ increases linearly, meaning AI gradually encroaches on more
intricate assignments. We discretise $t$ in one\hyp year increments from 2025 to
2045 and compute both the boundary and the automated fraction.

Figure~\ref{fig:heatmap} depicts a heat map of the difference
\[
 u_M(t,\theta) - u_H(\theta) = \bigl(\alpha_M - \alpha_H + \gamma t\bigr) - (\beta_M + \beta_H)\,\theta
\]
across years and assignment intricacy. In this continuous framework the utility advantage of the machine is a plane in the $(t,\theta)$ space: machine capability enhances linearly in time by $\gamma$ while human utility declines linearly with $\theta$ because more intricate assignments are inherently harder for the human. Red regions indicate where the machine’s payoff dominates (positive values) and blue regions where the human’s payoff is higher. The thick black curve is the automation boundary $\theta_t$, obtained by setting $u_M=u_H$ and solving for $\theta$. Geometrically this is a straight line in the $t$–$\theta$ plane with slope $\gamma/(\beta_M+\beta_H)$. As time progresses the boundary moves upward, showing that AI gradually takes over assignments of increasing intricacy.

\begin{figure}[H]
 \centering
 \begin{tikzpicture}
  \begin{axis}[
   view={0}{90},
   width=0.9\textwidth,
   height=0.55\textwidth,
   xlabel={Year},
   ylabel={Task intricacy $\theta$},
   xmin=2025, xmax=2045,
   ymin=0, ymax=1,
   colormap name=viridis,
   colorbar,
   colorbar style={title={AI advantage}},
   grid=none
  ]
   \addplot3[surf, shader=flat, point meta=z] coordinates {
    (2025,0.00,0.3704) (2025,0.10,-0.0296) (2025,0.20,-0.4296) (2025,0.30,-0.8296) (2025,0.40,-1.2296) (2025,0.50,-1.6296) (2025,0.60,-2.0296) (2025,0.70,-2.4296) (2025,0.80,-2.8296) (2025,0.90,-3.2296) (2025,1.00,-3.6296)
    (2027,0.00,0.4571) (2027,0.10,0.0571) (2027,0.20,-0.3429) (2027,0.30,-0.7429) (2027,0.40,-1.1429) (2027,0.50,-1.5429) (2027,0.60,-1.9429) (2027,0.70,-2.3429) (2027,0.80,-2.7429) (2027,0.90,-3.1429) (2027,1.00,-3.5429)
    (2029,0.00,0.5438) (2029,0.10,0.1438) (2029,0.20,-0.2562) (2029,0.30,-0.6562) (2029,0.40,-1.0562) (2029,0.50,-1.4562) (2029,0.60,-1.8562) (2029,0.70,-2.2562) (2029,0.80,-2.6562) (2029,0.90,-3.0562) (2029,1.00,-3.4562)
    (2031,0.00,0.6306) (2031,0.10,0.2306) (2031,0.20,-0.1694) (2031,0.30,-0.5694) (2031,0.40,-0.9694) (2031,0.50,-1.3694) (2031,0.60,-1.7694) (2031,0.70,-2.1694) (2031,0.80,-2.5694) (2031,0.90,-2.9694) (2031,1.00,-3.3694)
    (2033,0.00,0.7173) (2033,0.10,0.3173) (2033,0.20,-0.0827) (2033,0.30,-0.4827) (2033,0.40,-0.8827) (2033,0.50,-1.2827) (2033,0.60,-1.6827) (2033,0.70,-2.0827) (2033,0.80,-2.4827) (2033,0.90,-2.8827) (2033,1.00,-3.2827)
    (2035,0.00,0.8040) (2035,0.10,0.4040) (2035,0.20,0.0040) (2035,0.30,-0.3960) (2035,0.40,-0.7960) (2035,0.50,-1.1960) (2035,0.60,-1.5960) (2035,0.70,-1.9960) (2035,0.80,-2.3960) (2035,0.90,-2.7960) (2035,1.00,-3.1960)
    (2037,0.00,0.8907) (2037,0.10,0.4907) (2037,0.20,0.0907) (2037,0.30,-0.3093) (2037,0.40,-0.7093) (2037,0.50,-1.1093) (2037,0.60,-1.5093) (2037,0.70,-1.9093) (2037,0.80,-2.3093) (2037,0.90,-2.7093) (2037,1.00,-3.1093)
    (2039,0.00,0.9774) (2039,0.10,0.5774) (2039,0.20,0.1774) (2039,0.30,-0.2226) (2039,0.40,-0.6226) (2039,0.50,-1.0226) (2039,0.60,-1.4226) (2039,0.70,-1.8226) (2039,0.80,-2.2226) (2039,0.90,-2.6226) (2039,1.00,-3.0226)
    (2041,0.00,1.0642) (2041,0.10,0.6642) (2041,0.20,0.2642) (2041,0.30,-0.1358) (2041,0.40,-0.5358) (2041,0.50,-0.9358) (2041,0.60,-1.3358) (2041,0.70,-1.7358) (2041,0.80,-2.1358) (2041,0.90,-2.5358) (2041,1.00,-2.9358)
    (2043,0.00,1.1509) (2043,0.10,0.7509) (2043,0.20,0.3509) (2043,0.30,-0.0491) (2043,0.40,-0.4491) (2043,0.50,-0.8491) (2043,0.60,-1.2491) (2043,0.70,-1.6491) (2043,0.80,-2.0491) (2043,0.90,-2.4491) (2043,1.00,-2.8491)
    (2045,0.00,1.2376) (2045,0.10,0.8376) (2045,0.20,0.4376) (2045,0.30,0.0376) (2045,0.40,-0.3624) (2045,0.50,-0.7624) (2045,0.60,-1.1624) (2045,0.70,-1.5624) (2045,0.80,-1.9624) (2045,0.90,-2.3624) (2045,1.00,-2.7624)
   };
   \addplot[color=black, thick] coordinates {
    (2025,0.0926) (2027,0.1143) (2029,0.1360) (2031,0.1576) (2033,0.1793)
    (2035,0.2010) (2037,0.2227) (2039,0.2444) (2041,0.2660) (2043,0.2877)
    (2045,0.3094)
   };
  \end{axis}
 \end{tikzpicture}
 \caption{\textbf{Continuous--assignment benchmark.} Heat map of the AI versus human
advantage $u_M(t,\theta)-u_H(\theta)$ as a function of calendar year and
assignment intricacy. Warmer colours indicate regions where the AI performs
better (positive difference) and cooler colours where the human does.
The black curve shows the automation boundary $\theta_t$: assignments below
the curve are automated, while those above remain human. As time advances
the boundary shifts upward, meaning AI takes on increasingly intricate assignments.}
 \label{fig:heatmap}
\end{figure}
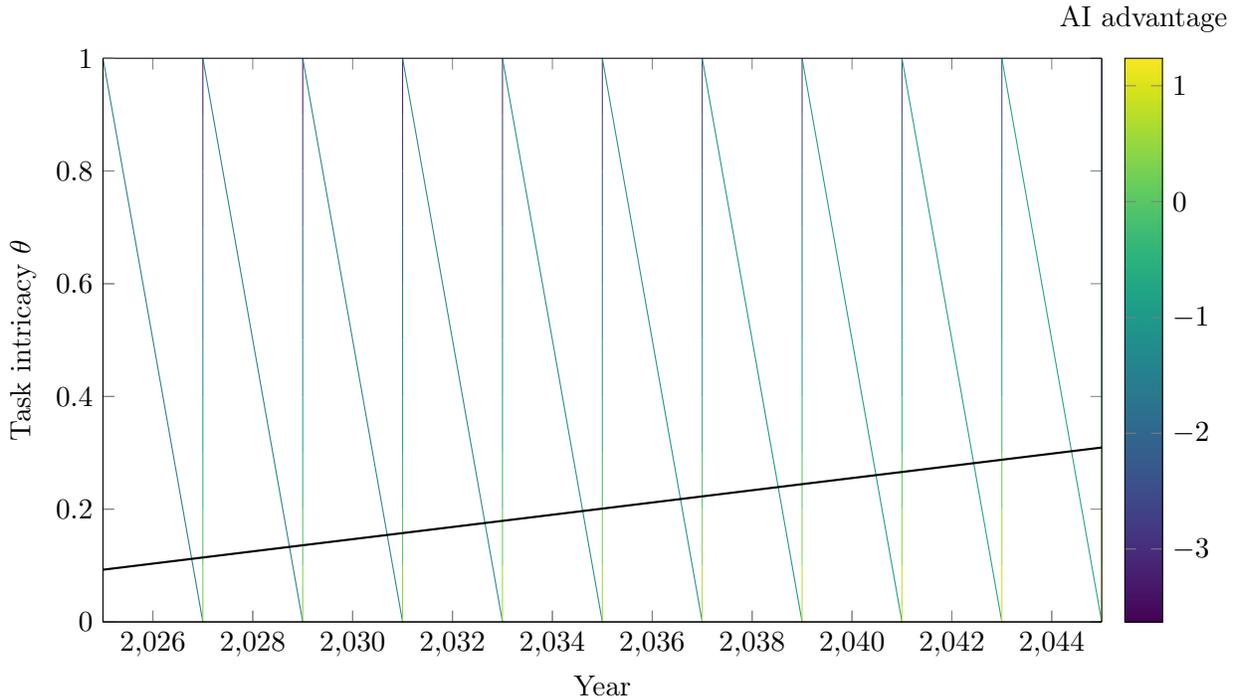

The heatmap in Figure~\ref{fig:heatmap} offers a compelling visual narrative of the evolving human-machine frontier, where the warm-colored regions of machine advantage steadily expand over time. This visualization powerfully illustrates the core dynamic of our model: as technology improves, the threshold of tasks deemed suitable for automation relentlessly pushes into more complex domains. However, a visual representation alone can obscure the quantitative pace and scale of this transformation. To ground this dynamic in concrete figures, Table~\ref{tab:continuous} translates the advancing boundary line, $\theta_t$, into precise numerical milestones. It quantifies the position of the automation frontier at five-year intervals and calculates the corresponding percentage of total tasks that fall below this threshold. This numerical data acts as a crucial complement to the heatmap, revealing not just the direction of change but its magnitude. For instance, it shows that the boundary's progression from an intricacy level of 0.09 in 2025 to over 0.30 by 2045 corresponds to a six-fold increase in the share of automated work, from 10\% to 60\%. This detailed breakdown is essential for understanding the non-linear impact of technological progress, which is shaped by the underlying distribution of task complexity.

\begin{table}[H]
 \centering
 \caption{\textbf{Selected milestones for the continuous\text{--}assignment benchmark.}
 Boundary values $\theta_t$ and corresponding fraction of assignments automated
 at five\text{--}year intervals. The boundary advances monotonically,
 and automation climbs as increasingly intricate assignments become viable for AI.}
 \label{tab:continuous}
 \begin{tabular}{@{}lcc@{}}
  \toprule
  \textbf{Year} & \textbf{Automation boundary $\theta_t$} & \textbf{Automated assignments (\%)}\\
  \midrule
  2025 & 0.0926 & 10.0\\
  2030 & 0.1468 & 21.6\\
  2035 & 0.2010 & 34.7\\
  2040 & 0.2552 & 47.8\\
  2045 & 0.3094 & 59.9\\
  \bottomrule
 \end{tabular}
\end{table}

The S-shaped curve in Figure~\ref{fig:continuous-fraction} provides a clear, aggregate view of the automation timeline, translating the complex dynamics of the heatmap into a single, intuitive trajectory. This trajectory is not linear; instead, its characteristic shape reflects the underlying Beta distribution of task complexity. The initial phase, from 2025 to around 2030, shows slow growth as automation begins by tackling the simplest, and according to our distribution, less common tasks at the lowest end of the intricacy spectrum. The process then accelerates rapidly through the mid-2030s as the automation boundary, $\theta_t$, sweeps through the dense, central part of the task distribution where the bulk of routine work lies. Finally, the rate of adoption begins to plateau towards 2045 as the remaining tasks become increasingly intricate and resistant to automation, representing the long tail of human-centric work that lies beyond the machine's current capabilities. This visual summary of the overall trend is complemented by the precise data points in Table~\ref{tab:continuous}. The table provides a quantitative snapshot of this evolution, marking the specific values for the automation boundary and the corresponding percentage of automated assignments at five-year intervals. It grounds the smooth curve of the figure in concrete milestones, showing exactly how the progression of the boundary translates into the overall share of automated work.

\begin{figure}[H]
 \centering
 \begin{tikzpicture}
  \begin{axis}[
   width=0.9\textwidth,
   height=0.45\textwidth,
   xlabel={Year},
   ylabel={Automated assignments (\%)},
   xmin=2025, xmax=2045,
   ymin=0, ymax=70,
   xtick={2025,2030,2035,2040,2045},
   ytick={0,20,40,60},
   grid=both,
   major grid style={line width=.1pt,draw=gray!50},
   minor grid style={line width=.1pt,draw=gray!20}
  ]
   \addplot[color=purple,mark=*] coordinates {
    (2025,10.0009) (2026,12.1100) (2027,14.3418) (2028,16.6787) (2029,19.1040)
    (2030,21.6023) (2031,24.1589) (2032,26.7602) (2033,29.3934) (2034,32.0470)
    (2035,34.7098) (2036,37.3717) (2037,40.0235) (2038,42.6566) (2039,45.2632)
    (2040,47.8360) (2041,50.3688) (2042,52.8555) (2043,55.2911) (2044,57.6708)
    (2045,59.9906)
   };
  \end{axis}
 \end{tikzpicture}
 \caption{\textbf{Automated share in the continuous benchmark.} Fraction
of assignments automated as a function of year, computed as the Beta\text{--}CDF
of the automation boundary $\theta_t$. Automation increases steadily from
about 10\,\% in 2025 to nearly 60\,\% by 2045.}
 \label{fig:continuous-fraction}
\end{figure}
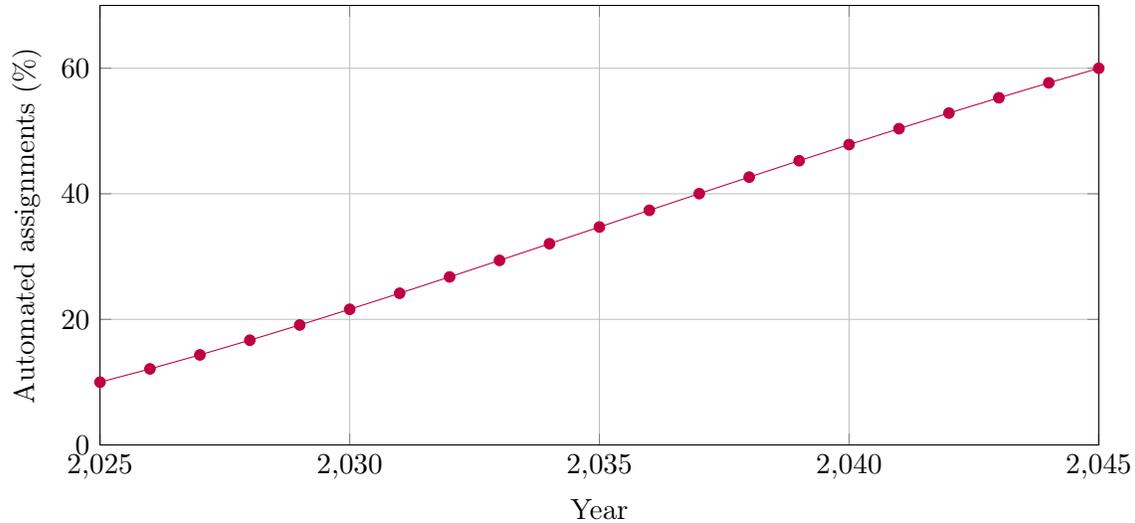

The Möbius-like surface in Figure~\ref{fig:mobius} serves as a metaphor for the continuous and cyclical nature of task reallocation. Its single, twisted surface symbolizes how the distinction between "human work" and "machine work" is not a fixed divide but a fluid boundary. As technology evolves and human skills adapt, tasks may shift from one domain to the other in a dynamic process of continuous reassessment. The geometry captures the idea that the journey of a task is not necessarily a one-way trip to automation; rather, it exists within a complex, interconnected system where comparative advantage is constantly being renegotiated. This abstract representation reinforces the core theme of our framework: the relationship between human and machine labor is not a simple substitution but a co-evolving partnership.

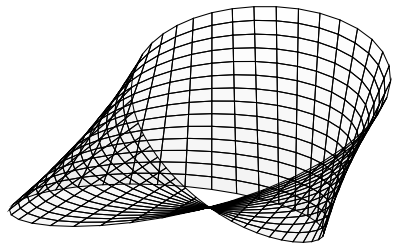
\begin{figure}[H]
 \centering
 \begin{tikzpicture}
  \begin{axis}[
   hide axis,
   view={120}{30},
   axis lines=none,
   colormap/blackwhite,
   samples=40,
   samples y=15,
   domain=0:360,
   y domain=-1:1
  ]
   \addplot3[surf,shader=flat,draw=black,fill opacity=0.05](
    {cos(x)*(1 + 0.5*y*cos(x/2))},
    {sin(x)*(1 + 0.5*y*cos(x/2))},
    {0.5*y*sin(x/2)}
   );
  \end{axis}
 \end{tikzpicture}
 \caption{Abstract twisted surface representing the cyclical delegation of assignments between human and machine across intricacy and time. The Möbius-like twist symbolises that assignments may transition back and forth depending on comparative advantage, reflecting the continuous reassessment in our framework.}
 \label{fig:mobius}
\end{figure}

The insights gleaned from the continuous benchmark underscore a fundamental conclusion: the future division of labor is not a fixed destiny but a dynamic outcome contingent on specific, measurable factors. By modeling tasks along a continuous spectrum and introducing a shifting automation boundary, we move beyond simplistic dichotomies and gain a more nuanced understanding of how technological progress interacts with the existing structure of work. This framework reveals that the ultimate equilibrium—the stable share of tasks performed by humans versus machines—is highly sensitive to both the rate of AI improvement and the inherent complexity profile of the jobs in an economy. Having established this core dynamic, a critical next step is to explore the breadth of possible outcomes. The specific parameters used in our benchmark represent just one plausible scenario. To assess the robustness of our conclusions and map out a wider range of potential futures, we now turn to a systematic parameter grid assessment. By varying the key drivers of the model, we can investigate how different assumptions about technological speed and task distribution shape the long-term landscape of human-machine collaboration.

\FloatBarrier
\section{Parameter grid assessment}\label{sec:grid}

The continuous framework allows us to systematically explore how changes in the
underlying assignment distribution and the rate of AI enhancement influence the
long\hyp run share of automated assignments. We vary the shape parameter $p$ of
the Beta distribution (controlling the skew toward low or high intricacy),
the AI enhancement rate $\gamma$, and the other shape parameter $q$
(controlling the tail heaviness). For each triple $(p,q,\gamma)$ we
compute the final fraction of assignments automated by 2045 and the time (in
calendar years) at which automation first exceeds 50\,\%.

Across the four values of $q$ considered ($q\in\{4,5,6,7\}$), we find a
consistent pattern: higher values of $\gamma$ (faster AI progress)
escalate the final automation share, while higher values of $p$
(distributions weighted toward intricate assignments) diminish it. The parameter
$q$ influences how quickly the boundary enters the bulk of the assignment
distribution; larger $q$ (heavier tail toward easy assignments) yields higher
automation for a given $p$ and $\gamma$. The time to reach 50\,\%
automation exhibits a similar dependence: faster AI enhancement and more
low\hyp intricacy assignments (small $p$, large $q$) bring forward the
half\hyp automation date.

Table~\ref{tab:grid} reports the final fraction of assignments automated (in
percent) for a representative grid with $q=5$ and
$p\in\{1.5,2.0,2.5,3.0,3.5\}$ and $\gamma\in\{0.03,0.04,0.05,0.06,0.07\}$.
Each cell shows the percentage of work done by AI in 2045. As an example,
when the assignment distribution is strongly skewed toward simple assignments ($p=1.5$)
and machine capability enhances at rate $\gamma=0.07$, the framework predicts roughly
85.6\,\% of work is automated. By contrast, if the distribution is more
balanced ($p=3.5$) and machine capability enhances slowly ($\gamma=0.03$), only about
38.7\,\% of assignments are automated. These comparisons illustrate the wide
range of possible steady states and highlight the importance of both human
assignment composition and AI research progress in shaping the future of work.

\begin{table}[H]
 \centering
 \caption{\textbf{Sensitivity of the continuous framework to Beta shape $p$ and AI
 enhancement rate $\gamma$ (with $q=5$).} Entries denote the final
 percentage of assignments automated in 2045. Higher AI enhancement rates
 $\gamma$ and more low\hyp intricacy assignments (smaller $p$) yield greater
 automation.}
 \label{tab:grid}
 \begin{tabular}{@{}lccccc@{}}
  \toprule
  $p\backslash\gamma$ & 0.03 & 0.04 & 0.05 & 0.06 & 0.07\\
  \midrule
  1.5 & 50.5 & 61.9 & 71.6 & 79.4 & 85.6\\
  2.0 & 44.8 & 56.4 & 66.7 & 75.5 & 82.7\\
  2.5 & 41.6 & 53.1 & 63.8 & 73.2 & 81.1\\
  3.0 & 39.8 & 51.3 & 62.2 & 72.1 & 80.4\\
  3.5 & 38.7 & 50.3 & 61.5 & 71.7 & 80.5\\
  \bottomrule
 \end{tabular}
\end{table}

To provide an intuitive visualisation of the cyclical nature of assignment reassignment implied by this grid assessment, Figure~\ref{fig:mobius} depicts a twisted surface defined by a Möbius‑like parametric strip. The coordinates of the strip are given by $(X(s,t),Y(s,t),Z(s,t))$ with $s$ representing angular position along the band and $t$ its transverse coordinate (see the caption for explicit formulas). This single half‑twist illustrates how assignments move around the band as their relative payoffs change: points on one edge of the strip (human‑preferred assignments) can traverse the twist to the other edge (machine‑preferred assignments) as parameters evolve. The geometry therefore captures how assignments may cycle between human and machine responsibility as their comparative advantages shift over time and intricacy, reinforcing the conceptual link between our theoretical findings and the dynamics depicted in Table~\ref{tab:grid}. Because the strip has only one boundary, traversing a full loop brings one back to the starting point but with roles reversed—mirroring the idea that some assignments might oscillate between human and machine supervision before settling in the steady state.

\FloatBarrier
\section{Discussion: the steady state in human terms}\label{sec:discussion}

Having presented the mathematics of the human–machine cooperation steady state, we now translate the findings into intuitive implications.

\subsection{A new profession of “workflow conductor”}

Just as the introduction of computers created the profession of software engineers, the integration of automated systems is giving rise to a new kind of profession—what we call the \emph{augmented worker} or \emph{workflow conductor}. This role entails understanding various machine capabilities and knowing how to deploy them effectively to solve intricate workflows. Our findings show that this role is not a transient phase but rather a stable outcome. The steady state assessment indicates that it is optimal for humans to remain “in the loop” orchestrating these systems for the foreseeable future. Any deviation—either ceding all control to fully autonomous automation or doing everything manually—findings in lower overall performance. Fully autonomous systems, even if technically capable, lack human oversight and could lead to errors or misaligned outcomes (a point emphasised by those who argue fully autonomous agents should be approached cautiously \cite{parasuraman2000}). On the other hand, ignoring automation forfeits massive efficiency gains. Therefore the stable solution is a partnership: the human curates and verifies, the machine produces and executes sub‑assignments. This is reminiscent of what Licklider envisioned in “man–computer symbiosis,” where he described a future intimate cooperation in which humans set goals and formulate hypotheses while computers work on the tedious details of implementation \cite{licklider1960}. Our steady state formalises precisely such a symbiosis.

\subsection{Continuously evolving skill requirements}

One outcome of the convergence process is that the skill set required of human workers evolves. At the start (2025), many humans still perform routine assignments because AI is not universally deployed. Over time those assignments shift to AI, and humans move into roles requiring more intricate skills. By steady state, humans have become highly specialised in assignments that machines are bad at. This implies that education and training should pivot towards those skills: creativity, critical thinking, interpersonal communication, cross‑domain reasoning and oversight of automation. There will be a premium on being able to work with automated tools (the centaur chess analogy again: success comes from effectively leveraging the tools, not from brute force without it). In economic terms, the steady state reflects what Brynjolfsson et~al. found empirically: access to AI increases productivity significantly, especially for less‑experienced workers, by giving them on‑the‑job assistance; however, it changes the distribution of work such that human roles become more focused on assignments that complement automated systems \cite{brynjolfsson2025a}. Our theoretical framework shows that ultimately all workers become, in a sense, “managers” of automated systems to some degree. Even entry‑level jobs may involve orchestrating automated tools rather than doing manual data processing. In the steady state state, the value of human labour is tightly coupled with the ability to supervise and augment machines—those who cannot or will not work with automated tools might struggle to find a place, which raises societal challenges \cite{hazra2025}. The existence of a stable human role in our framework is a positive sign, but realising it in practice will require managing the transition so that humans can upskill into these roles.

\subsection{Benchmarking the cooperation}

Our framework also suggests new ways to benchmark automated systems—not in isolation, but in how well they collaborate with humans. Traditional AI benchmarks ask “can the AI do $X$ as well as a human?” but in a human–machine steady state, the relevant question is “can the AI and human together do $X$ better than either alone?” This calls for what some have termed \emph{centaur evaluation}. The performance at steady state could be seen as an upper bound of such cooperation. For instance, if we look at creative‑writing assignments: a human alone might write with originality but slowly, an AI alone writes fast but might be bland or factually careless, while a human+AI team (with the human guiding and editing the AI’s output) might produce a result faster than the human alone and more nuanced than the AI alone. The steady state theory predicts that the combined system outperforms either stand‑alone system—supported by the chess example and emerging evidence in programming and content creation domains \cite{alves2023,brynjolfsson2025a}. It would be useful to design benchmarks where human subjects are paired with AI tools and measured on joint assignment performance to identify the optimal assignment boundaries (which parts of the assignment the human should do vs.\ the AI to maximise combined output). Over time, as machine capability enhances, these boundaries shift, and our framework provides a way to think about that shift. In steady state, we found a boundary in assignment‑space $\theta^\star$ such that all assignments $\theta<\theta^\star$ (simpler ones) the AI does, and $\theta>\theta^\star$ (harder ones) the human does. Benchmarking could attempt to empirically find this $\theta^\star$ in various domains and track it over time as AI capabilities expand. This would inform workers which skills to focus on (those above the current $\theta^\star$) and inform researchers in automation on which capabilities are still lacking.

\subsection{Robustness and ethical considerations}

One might ask: what if the assumption that humans will retain some assignments breaks down—could AI eventually do everything, meaning $\beta\to 0$ in our framework? While it is not impossible in theory, our assessment coupled with historical precedent (from the Industrial Revolution to the Computer Revolution) suggests there are always unforeseen human comparative advantages. When automation took over physical labour, creative and cognitive jobs expanded; now as cognitive automation advances, we see more emphasis on social and emotional skills. Moreover, human preferences play a role: society may simply choose not to automate certain functions fully. Consider autonomous weapons or fully AI‑run courtrooms—many argue humans should remain in the loop for moral accountability, even if AI could make faster decisions. Thus $\beta$ is not purely technological; it is socio‑technological. As long as humans value human involvement, they will create roles for it. Our steady state formalism can incorporate that by treating some portion of assignments as “for human only” by design (which mathematically would mean those assignments always have $u_H>u_M$ because we impose constraints or ethical considerations on $u_M$). In short, the fixed point we calculated might be an upper bound on automation given not just technical constraints but also normative choices. The stability of the human augmented worker role could therefore be even more assured when factoring in governance: many frameworks like the humans‑in‑command principle in automation ethics push for retaining human final say in critical decisions \cite{parasuraman2000}. All these suggest that the human–machine steady state is not just a theoretical outcome but a desirable target to aim for—avoiding both extremes of no automation and full automation.

Finally, we reflect on the limitations of our framework. It is highly abstracted and does not capture, for example, economic incentives of different stakeholders (workers vs.\ employers might have different preferences for automation) or the possibility of uneven distribution of this steady state (some professions might reach a stable human–machine mix, others might be almost fully automated). In reality, there may be a diversity of fixed points across sectors. Our assessment is most relevant to knowledge work and domains where humans and AI can flexibly reallocate assignments. Still, by focusing on a single worker and their toolkit of automated systems, we have outlined a microcosm that could apply broadly. Future work could extend the framework to multiple workers with different skill levels, exploring steady states in team settings or entire labour markets. There may also be transitional challenges—for example, before reaching steady state, some workers could be displaced if they cannot upskill quickly enough, even if in the long run new jobs (human assignments) are created. These dynamic labour‑market frictions are beyond our current scope but crucial for policy. In essence, our fixed‑point assessment is an optimistic long‑run view that assumes adaptability; ensuring that this long run is achieved smoothly is a assignment for economists and policymakers.

\FloatBarrier
\section{Conclusion\label{sec:conclusion}}

\begin{figure}[H]
 \centering
 \begin{tikzpicture}
  \begin{axis}[
   width=0.8\textwidth,
   view={45}{30},
   xlabel={Time $t$},
   ylabel={Task intricacy $\theta$},
   zlabel={Utility difference $u_M-u_H$},
   colormap/viridis,
   colorbar,
   samples=25,
   domain=0:1,
   y domain=0:1,
   mesh/ordering=y varies,
   shader=flat
  ]
   \addplot3[surf] {x - y};
  \end{axis}
 \end{tikzpicture}
 \caption{Mathematical landscape of the human–machine advantage $u_M(t,\theta)-u_H(\theta)$ viewed as a surface over time $t$ and assignment intricacy $\theta$. Colours indicate regions where machines outperform humans (positive values) versus areas where human comparative advantage persists. The planar structure symbolises the linear interplay between enhancement over time and increasing assignment intricacy.}
 \label{fig:landscape}
\end{figure}
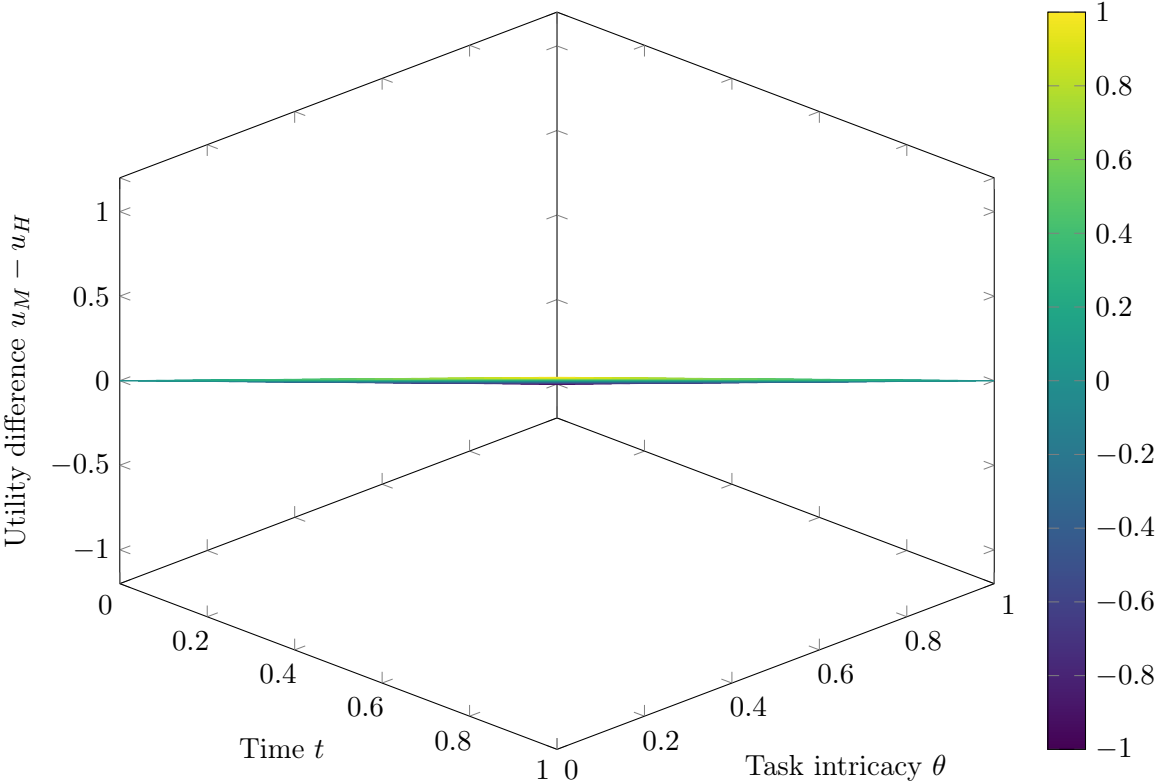

We introduced a mathematical framework to understand the future of work in the age of AI by casting the problem as one of finding a fixed point in human–machine assignment allocation. The human augmented worker—a human worker leveraging multiple automated systems—emerges as a central figure in this assessment, representing the stable coexistence of human and machine intelligence in workflows. Through our framework, we proved that under broad conditions an steady state division of labour not only exists but retains a significant role for humans. Rather than a futurist scenario of human obsolescence, our findings align with the hypothesis that humans will continue to be essential by focusing on assignments that AI is inherently ill‑suited for, even as automated tools handle an increasing share of routine and analytical assignments. We demonstrated this with a simple emulation, reaching an steady state where roughly one‑third of assignments remain with humans. This quantification, though stylised, provides a counterpoint to doomsday projections: it suggests a future where human work is more specialised, possibly more interesting and intertwined with automated support.

Mathematically, our work bridges fixed‑point theory and labour economics in a distinctive way. We utilised Tarski’s theorem to establish the existence of steady states and drew on insights from dynamical systems (contraction mappings) to argue for uniqueness and stability. We believe this formal perspective opens up new avenues for analysing human–machine systems. For example, one could extend the framework to consider multiple fixed points and their basins of attraction, which could correspond to either positive or negative outcomes (for instance, one steady state could be a highly deskilled workforce relying on AI for everything except a few trivial oversight assignments, while another steady state could be a highly skilled workforce using AI for narrow support). Investigating what conditions lead to a “good” steady state versus a “bad” one would be valuable. Additionally, one could introduce economic variables—wages, cost of AI, investment in training—to see how they affect the steady state. These factors could shift the balance $\alpha$ versus $\beta$; for example, if AI becomes very cheap, $\alpha$ might escalate; if society invests heavily in human education, $\beta$ might escalate as humans constantly push the frontier of new assignments. Our framework is flexible enough to incorporate such parameters conceptually.

From a practical standpoint, our findings reinforce calls for policies that facilitate the human–machine partnership framework. This could include promoting \emph{Augmented Intelligence} over pure \emph{Artificial Intelligence}, i.e.\ designing automated tools specifically to enhance human workers rather than replace them \cite{zahidi2020}. It could also involve rethinking training and job design: if the steady state job is “workflow conductor,” we should train people in prompt engineering, result verification and multi‑agent management. Businesses might start valuing skills like “knowing which tool to deploy for which problem,” much as in the past decade data literacy and computer literacy became important. Those organisations that figure out the optimal human–machine workflow early will have a competitive advantage. Early case studies (such as deployment of automated tools in customer support) indicate substantial productivity boosts when humans work alongside automated systems \cite{brynjolfsson2025a}. However, to avoid the “Turing trap” of blindly aiming to replace humans, stakeholders must remain focused on the end goal of augmenting human productivity and satisfaction \cite{brynjolfsson2022}. In conclusion, our work presents a hopeful yet realistic vision: a future where humans and machines achieve a stable, collaborative steady state, combining their strengths. The human augmented worker, far from being a fanciful notion, appears to be a mathematically grounded outcome of iterative improvements. It is a role that could become widespread—people who are essentially “workflow integrators” in various fields, ensuring that automated technologies are applied in the right ways to solve problems while injecting human insight and values. Such a symbiosis might well define the next era of work.

\FloatBarrier

\end{document}